\documentclass{article}

\PassOptionsToPackage{numbers, compress}{natbib}

    \usepackage[preprint]{neurips_2025}

\usepackage{color, colortbl}
\usepackage{xcolor} 
\definecolor{LightCyan}{rgb}{0.88,1,1}
\usepackage{caption}
\usepackage{subcaption}
\usepackage{multirow}
\usepackage{enumitem}
\usepackage{amsmath}
\usepackage{bbm}
\usepackage[most]{tcolorbox}
\usepackage{algorithm}
\usepackage{algpseudocode}
\usepackage{macros}
\usepackage{balance}
\usepackage{booktabs}
\usepackage{url}
\usepackage{hyperref}

\newtheorem{theorem}{Theorem}
\newtheorem{theoremdef}{Theorem}
\newtheorem{definition}[theoremdef]{Definition}
\newtheorem{definition*}{Problem}

\theoremstyle{remark}

\title{Debiased Negative Mining Improves Out-of-distribution Detection\\ with Pre-trained Vision-Language Models}

\author{Bo Peng, Jie Lu, Guangquan Zhang, Zhen Fang \\
University of Technology Sydney\\
}

\begin{document}

\maketitle

\begin{abstract}
Aiming at identifying unexpected inputs from unknown classes, out-of-distribution (OOD) detection has emerged as a pivotal approach to enhancing the reliability of machine learning models.
This paper focuses on the burgeoning paradigm of post-hoc OOD detection with pre-trained vision-language models (VLMs), where a popular pipeline is to detect OOD inputs by examining their affinities between ID labels and negative labels, i.e., those semantically different from ID labels.
Due to the unavailability of target OOD labels, existing works predominantly rely on heuristic rules to mine negative labels from unlabeled wild corpus data. 
Despite the empirical success, we argue that the power of VLM-based OOD detection has yet to be fully unleashed since the notorious false negative problem is far from addressed in the literature.
With this motivation, we are interested in addressing the challenge of mining true negative labels for OOD scoring.
To this end, we develop a theoretical framework for correcting the sampling bias of negatives labels by indirectly approximating the distribution of negative labels.
Perhaps surprisingly, we show that the debiased negative mining can be naturally converted into Monte-Carlo sampling based on ID labels and the unlabeled wild corpus data.
Extensive experiments empirically manifest that our method establishes a new state-of-the-art in a variety of OOD detection setups. Code is publicly available at \href{https://github.com/60pen9/Debiased-Negative-Mining-Improves-OOD-Detection-with-Pre-trained-VLMs}{\textcolor{red}{here}}.
\end{abstract}

\section{Introduction}
Despite the significant progress in machine learning that has facilitated a broad spectrum of real-world tasks~\cite{a,b,c,d,f,g,h,i,j,k,m,n,o,p,t,s}, existing models often operate under a \textit{closed-world} scenario, where test data stems from the same distribution as the training data. However, real-world applications often entail \textit{open-world} scenarios in which deployed models may encounter unseen classes of samples during training, giving rise to what is known as out-of-distribution (OOD) data. These OOD instances can potentially undermine a model's stability and, in certain cases, inflict severe damage on its performance. Accordingly, a reliable discriminative model should not only correctly classify known in-distribution (ID) samples but also flag any OOD inputs as unknown. This directly motivates OOD detection~\cite{3} which makes significant differences in ensuring the safety of decision-critical applications, e.g., autonomous driving~\cite{7}, medical diagnosis~\cite{8}, and cyber-security~\cite{9}.

Up to now, a plethora of OOD detection algorithms have been developed recently by leveraging post-hoc analysis on the pre-trained model. The seminal work~\cite{10} leverages the maximum softmax probability (MSP), also known as the softmax confidence score, for OOD detection, which is built upon the hypothesis that OOD data should trigger relatively lower softmax confidence than that of ID data. 
ODIN~\cite{11} extends MSP by using temperature scaling and input perturbation to amplify the ID/OOD separability. 
However, approaches of this category are challenged by the observation that deep neural networks are prone to produce over-confident predictions~\cite{12} even though the inputs are far away from the training data. As a result, advanced methods turn to design alternative OOD scoring functions by resorting to the stored ID patterns in gradients~\cite{13}, features~\cite{18,15,14,19,25,16,20} and logits~\cite{23,21,24,22}.

Recently, the empirical success of contrastive language-image pre-training (CLIP)~\cite{17} shifts research toward expanding post-hoc OOD detection from single-modal to multi-modal methods.
Researchers have since explored ways to better leverage multi-modal models to enhance the performance and applicability of post-hoc OOD detection. 
The pioneering work~\cite{26} introduces a trainable captioner to generate candidate OOD labels to match OOD images. However, when dealing with large-scale datasets, such as ImageNet-1k, that usually encompasses a large number of ID classes the trained captioner may not generate effective candidate
OOD labels, which leads to poor performance of OOD detection. Differently, MCM~\cite{27}, which defines textual features as concept prototypes for each ID class and uses the scaled distance between visual features and the closest ID prototype to measure OOD uncertainty. This method has paved the way for using pre-trained vision-language models (VLMs) in post-hoc OOD detection. However, MCM relies only on textual information from the ID label space, leaving VLMs’ text interpretation capabilities underutilized. To address this, NegLabel~\cite{28} selects negative labels from unlabeled lexical databases in the wild\footnote{Generally, “in-the-wild” data are those that can be collected almost for free upon deploying machine learning models in the open world.}(e.g., WordNet~\cite{29}), based on their similarities to the ID label space, which equips the model with stronger ability to distinguish OOD samples. 

Despite the empirical success, the power of VLM-based OOD detection has yet to be fully unleashed even with the helm of conjugated semantic pool~\cite{30}. This is largely because the advanced negative mining strategy~\cite{28} tends to rely on heuristic rules to mine negative labels from unlabeled wild corpus data, where cosine similarity is used to quantify semantic disagreement with ID labels. Albeit straightforward and intuitive, there is no theoretical guarantee that cosine similarity perfectly corresponds to true similarity or dissimilarity measure. As a result, the mined negative labels could be inevitably contaminated by false negative labels. This phenomenon, referred to as negative mining bias, leads to a biased OOD scoring function and accordingly sub-optimal OOD detection performance. The discussion above directly prompts the following non-trivial yet under-explored question: \textit{Is it possible to mitigate negative mining bias without requiring any human feedback (positive or negative) on the wild corpus data?} 

In this paper, we give an affirmative answer to this question without extra supervision. To be specific, our key idea lies in indirectly approximating the distribution of negative labels, which is based on the ideas from positive-unlabeled learning~\cite{31,32,33}. This enables us to efficiently design the negative mining strategy as an easy-to-implement importance sampling technique that proceeds with only access to the unlabeled wild corpus data and ID labels. 
Extensive experiments empirically demonstrate that our method establishes state-of-the-art performance.

The contribution of this work is summarized as follows:
\begin{itemize}[leftmargin=1em]
    \item We develop a novel method that corrects for the mining bias of negative labels without requiring any ground-truth similarity or dissimilarity information.
    \item We implement our method as a practical importance sampling strategy while only assuming access to positive examples and the unlabeled data.
    \item Extensive experiments and related analysis on multiple OOD detection benchmarks with state-of-the-art performances, which demonstrate the effectiveness of our method.
\end{itemize}


\section{Preliminary}
\textbf{Notation.} We write matrices and vectors as bold-faced uppercase and lowercase characters respectively. Let $\mathcal{X}$ and $\mathcal{Y}$ represent the visual input space and the label space, respectively. Given a random variable $Y \in \mathcal{Y}$, we write $\mathbb{P}_Y$ as the marginal distribution defined over $\mathcal{Y}$, and use $y\sim\mathbb{P}_Y$ to indicate a samples drawn from $\mathbb{P}_Y$. We also slightly abuse the notification to use $\left\{y_1,\ldots,y_n\right\}\iidsim\mathbb{P}_Y$ to represent a $n$-sized set of i.i.d. samples drawn from $\mathbb{P}_Y$.

\textbf{Problem Setup.} Considering $K$-way classification as a case study, we use $\mathcal{Y}_{\mathrm{I}}\triangleq\{y_1,\ldots,y_K\}\subset\mathcal{Y}$ as the space for \textit{known} ID labels\footnote{Significantly different from traditional OOD detection, this paper defines labels as words, e.g. class names, rather than numbers, e.g., class indexes.}. The joint ID distribution, represented as $\mathbb{P}_{X_{\mathrm{I}} Y_{\mathrm{I}}}$, is a joint distribution defined over $\mathcal{X} \times \mathcal{Y}_{\mathrm{I}}$. During testing, there are some unknown OOD joint distributions $\mathbb{P}_{X_{\rm o}Y_{\rm o}}$ defined over $\mathcal{X}\times\mathcal{Y}_{\mathrm{o}}$, where $\mathcal{Y}_{\mathrm{o}}\subseteq\mathcal{Y}\setminus\mathcal{Y}_{\mathrm{I}}$ presents the space of \textit{unknown} OOD labels.

\textbf{Post-hoc OOD Scoring.} Existing methods~\cite{13,14,15,21,22} tend to adopt a post-hoc strategy to detect OOD samples, \textit{i.e.,} given a pre-trained ID classification model $f$ and a scoring function $S(\cdot;f):\mathcal{X}\rightarrow\mathbb{R}$, then $\mathbf{x}$ is detected as ID data if and only if $S(\mathbf{x};f)\geq \beta$, for some given threshold $\beta$:
\begin{equation}\label{eq1}
\omega(\mathbf{x})= \text{ID},~\text{if}~S(\mathbf{x};f) \geq \beta;~ \text{otherwise},~\omega(\mathbf{x})=\text{OOD}.
\end{equation}
Typically, $\beta$ is chosen to ensure a high fraction (e.g., 95\%) of ID data to be correctly classified, which is independent of OOD data.

\textbf{CLIP-based Models} adopt a dual-stream architecture~\cite{17} with one text encoder $f_{\mathcal{T}}$ and one image encoder $f_{\mathcal{X}}$ to map inputs of two modalities into a $(d-1)$-dimension uni-modal hyper-spherical space $\mathbb{S}^{d-1}\triangleq\left \{\mathbf{z}\in\mathbb{R}^{d}|\left \|\mathbf{z}\right \|_2=1\right \}$. Zero-shot image classification for a pre-trained CLIP-like model is to classify images into one of known ID classes without fine-tuning by computing $\arg\max_{j=1,\ldots K}h(\mathbf{x},y_j)$ where $h(\mathbf{x},y_j)\triangleq \kappa\cdot f_{\mathcal{X}}(\mathbf{x})^\top f_{\mathcal{T}}\big(\Delta(y_j)\big)$ where $\kappa>0$ is a temperature and $\Delta(\cdot)$ produces the text prompt for the input label. 

\textbf{CLIP-based OOD Detection with Negative Labels.} CLIP-based models, thanks to their remarkable effectiveness~\cite{17} and provable guarantees~\cite{35}, are recently extended to the task of zero-shot OOD detection where there is no need to train on ID samples. A popular pipeline of this task is to filter a $L$-sized set of negatives labels\footnote{By definition, \textit{negative} labels are those semantically \textit{irrelevant/dissimilar} to \textit{all} ID labels while \textit{positive} labels are those semantically \textit{relevant/similar} to \textit{any} ID label.} $\{y_{K+1},\ldots,y_{K+L}\}$ from unlabeled wild lexical database like WordNet~\cite{29} to formulate the OOD scoring of $\mathbf{x}$ in Eq. (\ref{eq1}) as the model’s prediction confidence that $\mathbf{x}$ belongs to $\mathcal{Y}_{\mathrm{I}}$, i.e., 
\begin{equation}
\label{NegLabel}
\begin{split}
S_{\text{NegLabel}}(\mathbf{x};f)  &\triangleq \frac{1}{B}\sum_{b=1}^B \frac{\sum_{y\in\mathcal{Y}_{\rm I}}e^{h(\mathbf{x},y)}}{\sum_{y\in\mathcal{Y}_{\rm I}}e^{h(\mathbf{x},y)}+\sum_{\hat{y}\in\mathcal{G}_b}e^{h(\mathbf{x},\hat{y})}},
\end{split}
\end{equation}
where, as suggested by NegLabel~\cite{28}, the selected $L$ negative labels is randomly grouped into $B$ \textit{non-overlapping} subsets $\mathcal{G}_1,...,\mathcal{G}_B$ such that $\{y_{K+i}\}_{i=1}^L=\bigcup_{b=1}^B\mathcal{G}_b$ with ${\mathcal{G}}_i\bigcap{\mathcal{G}}_j=\emptyset, \forall i \neq j$.
\section{Motivation}
Our motivation starts from exploring the ideal case of the negative labels-guided OOD scoring function. The formal formulation, which we will consider as unbiased, is given as follows:
\begin{equation}
\label{unbiased}
S_{\text{unbiased}}(\mathbf{x};f)  \triangleq \mathbb{E}_{\left \{{\hat{y}_i}\right \}_{i=1}^r\,\iidsim\,\mathbb{P}_Y^{-}} {\left [\Phi\left (\mathbf{x}, \left \{\hat{y}_i\right \}_{i=1}^r \right )\right]},
\end{equation}
where $\mathbb{P}_Y^{-}$ is the true distribution of negative labels and 
\begin{equation}
\label{eq4}
\Phi\left (\mathbf{x}, \left \{\hat{y}_i\right \}_{i=1}^r \right ) \triangleq \frac{\sum_{i=1}^K e^{h(\mathbf{x},y_i)}}{\sum_{i=1}^K e^{h(\mathbf{x},y_i)}+\frac{\lambda}{r}\sum_{i=1}^re^{h(\mathbf{x},{\hat{y}_i})}}.
\end{equation}
Here, the weighting constant $\lambda>0$ in Eq. (\ref{eq4}) is introduced for the purpose of analysis only (we typically set $\lambda = r$ for finite $r$). Unfortunately, the ideal OOD scoring function $S_{\text{unbiased}}(\mathbf{x};f)$ in Eq. (\ref{unbiased}) is not achievable since the true negative distribution $\mathbb{P}_Y^{-}$ is not accessible in practice. Therefore, a natural and efficient solution is to explicitly find a surrogate negative distribution $\hat{\mathbb{P}}_Y^{-}$ from the off-the-shelf (unlabeled) wild distribution $\mathbb{Q}_Y^{-}$, i.e., 
\begin{equation}
\label{biased}
S_{\text{biased}}(\mathbf{x};f)  \triangleq \mathbb{E}_{\left \{{\hat{y}_i}\right \}_{i=1}^r\,\iidsim\,\hat{\mathbb{P}}_Y^{-}} {\left [\Phi\left (\mathbf{x}, \left \{\hat{y}_i\right \}_{i=1}^r \right )\right]}.
\end{equation}
One can easily check that $S_{\text{NegLabel}}(\mathbf{x};f)$ in Eq. (\ref{NegLabel}) is a Monte-Carlo estimator of $S_{\text{biased}}(\mathbf{x};f)$ in Eq. (\ref{biased}) with $r=L/B$. While $S_{\text{NegLabel}}(\mathbf{x};f)$ has been empirically found effective in detecting OOD inputs, $S_{\text{biased}}(\mathbf{x};f)$ is still considered to be biased since the surrogate negative distribution $\hat{\mathbb{P}}_Y^{-}$ is empirically approximated under a distance-based principle, i.e., true negative labels should be far from the ID labels of interest in the embedding space. To be specific, aside from the difficulties of choosing suitable distance functions, there is no theoretical guarantees that handcrafted metrics could correctly capture semantic relationships. This makes the distance-based principle suffer from  the notorious false negative problem to lead to $\hat{\mathbb{P}}_Y^{-}\ne\mathbb{P}_Y^{-}$ in most cases. This motivates our work.
\begin{tcolorbox}[center, width=120mm, colback=blue!5!white,colframe=blue!75!black,colbacktitle=red!80!black]
\textit{To reduce gap between $S_{\text{biased}}(\mathbf{x};f)$ and $S_{\text{unbiased}}(\mathbf{x};f)$, our key motivation is to develop a correction for the sampling bias caused by the false negative problem while still assuming only access to unlabeled wild data.}
\end{tcolorbox}
 
\section{Methodology}
Based on the discussion above, we are interest in deriving a debiased OOD scoring function without requiring explicitly sampling from the unknown true distribution of negative labels. Similar to prior works~\cite{28,30,40}, we approach this goal by leveraging unlabeled “in-the-wild” corpus data which can be collected almost for free in the open world. However, motivated by the observation that wild corpus data usually contains a mixture of positive and negative labels with regard to the ID labels of interest, the novelty of our method lies in decomposing the wild data distribution as follows:

\begin{definition}[Wild Data Distribution]
\label{D1}
Let $\mathbb{P}_Y^{+}$ and $\mathbb{P}_Y^{-}$ be the true distributions of positive and negative labels defined over $\mathcal{Y}$, respectively. According to the Huber contamination model~\cite{68}, we can model the unlabeled wild data distribution $\mathbb{Q}_Y$ as follows:
\begin{equation}
\label{eq5}
\mathbb{Q}_Y\triangleq\tau\cdot\mathbb{P}_Y^{+}+(1-\tau)\cdot\mathbb{P}_Y^{-},
\end{equation}
where $\tau\in(0,1)$ is a prior probability. As $\tau$ is often unknown in practice, we opt to treat it as a hyperparameter in this paper\footnote{While the prior $\tau$ can be estimated with a binary classifier in the literature of positive-unlabeled learning~\cite{31,32,33}, this requires labeled training data and therefore is not applicable in the content of zero-shot OOD detection}.
\end{definition}

\begin{definition}[Empirical Wild Dataset]
\label{D2}
An empirical wild corpus dataset $\mathcal{D}=\left\{\tilde{y}_1,\ldots,\tilde{y}_T\right\}$ is sampled from the wild data distribution $\mathbb{Q}_Y$ independently and identically distributed. 
\end{definition}

As a direct implication of Definition~\ref{D1}, we can rearrange Eq. (\ref{eq5}) to rewrite the unknown true negative distribution $\mathbb{P}_Y^{-}$ as follows:
\begin{equation}
\label{eq6}
    \mathbb{P}_Y^{-}=(\mathbb{Q}_Y-\tau\cdot\mathbb{P}_Y^{+})/(1-\tau). 
\end{equation}
However, simply replacing Eq. (\ref{eq6}) with $\mathbb{P}_Y^{-}$ in Eq. (\ref{unbiased}) would make the resulting OOD scoring function in Eq. (\ref{eq7}) computationally expensive especially for a large $r$:
\begin{equation}
\label{eq7}
S_{\text{unbiased}}(\mathbf{x};f) = \sum_{k=0}^r
\expectunder[\substack{
\{\hat{y}_i\}_{i=1}^k\,\iidsim\,\mathbb{P}_Y^{+}  \\ 
\{\hat{y}_i\}_{i=k+1}^{r}\,\iidsim\,\mathbb{Q}_Y
    }]
{\binom{r}{k}
\frac{(-\tau)^k}{(1-\tau)^r}\cdot\Phi\left (\mathbf{x}, \left \{\hat{y}_i\right \}_{i=1}^r\right )}.
\end{equation}
Please refer to Appendix~\ref{appendix b} for detailed deviation. In view of this, we proceed from a different perspective by investigating the asymptotic form of $S_{\text{unbiased}}(\mathbf{x};f)$ in Eq. (\ref{unbiased}) as $m\rightarrow+\infty$.
\begin{theorem} 
\label{thm1}
For fixed $\lambda>0$, as the number of negative labels $r\rightarrow+\infty$, the dominated convergence theorem implies that $S_{\text{unbiased}}(\mathbf{x};f)$ in Eq. (\ref{unbiased}) converges to
\begin{equation}
\label{eq8}
\begin{split}
\lim_{m\rightarrow+\infty}S_{\text{unbiased}}(\mathbf{x};f)
=\hat{S}_{\text{unbiased}}(\mathbf{x};f)\triangleq\frac{\sum_{i=1}^K e^{h(\mathbf{x},y_i)}}{\sum_{i=1}^K e^{h(\mathbf{x},y_i)}+\lambda\mathbb{E}_{\hat{y}\sim\mathbb{P}_Y^{-}}[e^{h(\mathbf{x},{\hat{y}})}]}.
\end{split}
\end{equation}
\end{theorem}

\begin{proof}
    Since $\Phi\left (\mathbf{x}, \left \{\hat{y}_i\right \}_{i=1}^m \right )$ is strictly bounded given a finite $m$ for any $\mathbf{x}$ and $\left \{\hat{y}_i\right \}_{i=1}^m\iidsim\,\mathbb{Q}_Y$, applying the dominated convergence theorem completes the proof:
\begin{align*}
    \lim_{m \rightarrow \infty} S_{\text{unbiased}}(\mathbf{x};f)
    = &\, \lim_{m \rightarrow \infty}\underset{\left \{{\hat{y}_i}\right \}_{i=1}^m\,\iidsim\,\mathbb{P}_Y^{-}}{\mathbb{E}} \left [ \frac{\sum_{i=1}^K e^{h(\mathbf{x},y_i)}}{\sum_{i=1}^K e^{h(\mathbf{x},y_i)}+\frac{\lambda}{m}\sum_{i=1}^me^{h(\mathbf{x},{\hat{y}_i})}} \right]\\
    =&\, \underset{\left \{{\hat{y}_i}\right \}_{i=1}^m\,\iidsim\,\mathbb{P}_Y^{-}}{\mathbb{E}} \left [\lim_{m \rightarrow \infty} \frac{\sum_{i=1}^K e^{h(\mathbf{x},y_i)}}{\sum_{i=1}^K e^{h(\mathbf{x},y_i)}+\frac{\lambda}{m}\sum_{i=1}^me^{h(\mathbf{x},{\hat{y}_i})}} \right]\\
    =&\, \underset{\left \{{\hat{y}_i}\right \}_{i=1}^m\,\iidsim\,\mathbb{P}_Y^{-}}{\mathbb{E}} \left [ \frac{\sum_{i=1}^K e^{h(\mathbf{x},y_i)}}{\sum_{i=1}^K e^{h(\mathbf{x},y_i)}+\lambda\mathbb{E}_{\hat{y}_i\sim\mathbb{P}_Y^{-}} e^{h(\mathbf{x},{\hat{y}_i})}} \right ] \\
    =&\,  \frac{\sum_{i=1}^K e^{h(\mathbf{x},y_i)}}{\sum_{i=1}^K e^{h(\mathbf{x},y_i)}+\lambda\mathbb{E}_{\hat{y}_i\sim\mathbb{P}_Y^{-}} e^{h(\mathbf{x},{\hat{y}_i})}}.
\end{align*}
\end{proof}

Intuitively speaking, $S_{\text{unbiased}}(\mathbf{x};f)$ in Eq. (\ref{unbiased}) can be viewed as a finite negative labels approximation to $\hat{S}_{\text{unbiased}}(\mathbf{x};f)$ in Eq. (\ref{eq8}).
Combining Eq. (\ref{eq6}) with Eq. (\ref{eq8}), we arrive at the following:
\begin{equation}
\label{eq9}
\begin{split}
    & \hat{S}_{\text{unbiased}}(\mathbf{x};f) \trieq \frac{\frac{1-\tau}{\lambda}\sum_{i=1}^K e^{h(\mathbf{x},y_i)}}{\frac{1-\tau}{\lambda}\sum_{i=1}^K e^{h(\mathbf{x},y_i)}+\mathbb{E}_{\hat{y}\sim\mathbb{Q}_Y}[e^{h(\mathbf{x},{\hat{y}})}]-\tau\mathbb{E}_{\hat{y}\sim\mathbb{P}_Y^{+}}[e^{h(\mathbf{x},{\hat{y}})}]},
\end{split}
\end{equation}
Clearly, Eq. (\ref{eq9}) is more computationally efficient than Eq. (\ref{eq7}) since what we only need to do is to approximate expectations $\mathbb{E}_{\hat{y}\sim\mathbb{Q}_Y}[e^{h(\mathbf{x},{\hat{y}})}]$ and $\mathbb{E}_{\hat{y}\sim\mathbb{P}_Y^{+}}[e^{h(\mathbf{x},{\hat{y}})}]$ respectively. Inspired by \citet{20}, this can be efficiently achieved by using the classical Monte-Carlo importance sampling technique~\cite{38}. 

To be specific, given $m$ i.i.d. samples $\{u_i\}_{i=1}^{m}$ from $\mathbb{Q}_Y$ and $n$ i.i.d. samples $\{v_i\}_{i=1}^{n}$ from $\mathbb{P}_Y^{+}$, the law of large numbers implies
\begin{equation}
\label{eq10}
\lim_{m\rightarrow\infty}\frac{1}{m}\sum_{i=1}^{m}e^{h(\mathbf{x},u_i)}=\mathbb{E}_{u\sim\mathbb{Q}_Y}[e^{h(\mathbf{x},{u})}],
\end{equation}
\begin{equation}
\label{eq11}
    \lim_{n\rightarrow\infty}\frac{1}{n}\sum_{i=1}^{n}e^{h(\mathbf{x},v_i)}=\mathbb{E}_{v\sim\mathbb{P}_Y^+}[e^{h(\mathbf{x},{v})}].
\end{equation}
With the help of Eq. (\ref{eq10}) and Eq. (\ref{eq11}), we can arrive at the following debiased OOD scoring function as a Monte-Carlo estimator of $\hat{S}_{\text{unbiased}}(\mathbf{x};f)$ in Eq. (\ref{eq9}):
\begin{equation}
\label{eq12}
\begin{split}
    & \hat{S}_{\text{debiased}}(\mathbf{x};f) \trieq \frac{\frac{1-\tau}{\lambda}\sum_{i=1}^K e^{h(\mathbf{x},y_i)}}{\frac{1-\tau}{\lambda}\sum_{i=1}^K e^{h(\mathbf{x},y_i)}+\frac{1}{m}\sum_{i=1}^{m}e^{h(\mathbf{x},u_i)}-\frac{\tau}{n}\sum_{i=1}^{n}e^{h(\mathbf{x},v_i)}},
\end{split}
\end{equation}
Similarly, we set $\lambda=m$ when $m$ is finite for simplicity, which is considered as the default setting in the rest of this paper. For completeness, we theoretically prove in Theorem~\ref{thm2} that the bias of $\log\hat{S}_{\text{debiased}}(\mathbf{x};f)$ relative to $\log\hat{S}_{\text{unbiased}}(\mathbf{x};f)$ is upper-bounded and decreases with the rate $O(m^{-1/2}+n^{-1/2})$.
\begin{theorem}
\label{thm2}
Let $\{u_i\}_{i=1}^{m}$ and $\{v_i\}_{i=1}^{n}$ be two sets of i.i.d. samples drawn from $\mathbb{Q}_Y$ and $\mathbb{P}_Y^{+}$ respectively. Further, let us define\footnote{Note that $\delta(\mathbf{x};f)$ also depend on both $\{u_i\}_{i=1}^{m}$ and $\{v_i\}_{i=1}^{n}$, but we suppress the dependence from the notation for brevity.}
\begin{equation}
\notag
\delta(\mathbf{x};f)=\big |\log\hat{S}_{\text{unbiased}}(\mathbf{x};f)-\log\hat{S}_{\text{debiased}}(\mathbf{x};f) \big |.
\end{equation}
Then for any $\mathbf{x}\in\mathcal{X}$, $f$, and $\zeta > 0$, we have
\begin{equation}
\nonumber
\mathbb{E}_{\{u_i\}_{i=1}^{m},\{v_i\}_{i=1}^{n}}
\left [\delta(\mathbf{x};f)\right ]\leq\frac{1}{1-\tau}\sqrt{\frac{\pi e^{3\kappa}}{2m}} + \frac{\tau}{1-\tau}\sqrt{\frac{\pi e^{3\kappa}}{2n}}
\end{equation}
\end{theorem}
\begin{proof}
See Appendix~\ref{appendix c} for the proof.
\end{proof}

\section{Implementation}
As shown in Eq. (\ref{eq12}), our proposed $\hat{S}_{\text{debiased}}(\mathbf{x};f)$ requires the construction of $\{u_i\}_{i=1}^{m}$ and $\{v_i\}_{i=1}^{n}$. In what follows, we show that our implementation, which can be divided into 3 steps, is highly efficient and simple, requiring minimal specialized operations.

\textbf{Step 1.} To construct $\{u_i\}_{i=1}^{m}$, a naive approach is to directly use the off-the-shelf wild dataset $\mathcal{D}$ (given by Definition~\ref{D2}). However, since wild data is is available in abundance~\cite{63}, the large-scale nature of $\mathcal{D}$ make it impractical to caching the entire $\mathcal{D}$. In response, we consider selecting $m$ the most representative samples in $\mathcal{D}$, where $m=L\ll T$ for simplicity and in alignment with prior works~\cite{28,30,40}. Intuitively, representative samples are expected to reside in dense regions. Thus, we measure the representativeness of a sample $\tilde{y}_i\in\mathcal{D}$ by the density of its $\alpha$-nearest neighbors, i.e.,
\begin{equation}
\label{eq14}
Rep(\tilde{y}_i) = -\log \sum_{\tilde{y}_j\in\mathcal{M}(\tilde{y}_i,\alpha)}\left \| \tilde{y}_i-\tilde{y}_j\right \|_2^2,
\end{equation}
where $\tilde{y}_j\in\mathcal{M}(\tilde{y}_i,\alpha)$ denotes denotes $\tilde{y}_j$ is among the $\alpha$ nearest neighbors of $\tilde{y}_i$ in $\mathcal{D}$. A higher $Rep(\tilde{y}_i)$ suggests a more compact local structure, indicating that the sample $\tilde{y}_j$ is more representative.
Accordingly, we reorder $\mathcal{D}$ by according to the decreasing representativeness measured by Eq. (\ref{eq14}). Denote the reordered dataset as $\hat{\mathcal{D}}=\left\{\tilde{y}_{(1)},\ldots,\tilde{y}_{(T)}\right\}$, the representativeness-based selection criteria is designed as $u_i=\tilde{y}_{(i)}, \forall i=1,\ldots L$.

\textbf{Step 2.} While it is challenging to sample positive labels to construct $\{v_i\}_{i=1}^{n}$ due to the unavailability of true knowledge about $\mathbb{P}_Y^{+}$, the definition of positive labels enables us to synthesize the feature of positive labels by applying a semantics-preserving transformation $\Omega(\cdot)$ on each ID label in the embedding space. In particular, with $n=K$, we take $f_{\mathcal{T}}\big(\Delta(v_i)\big)=\Omega\big(f_{\mathcal{T}}\big(\Delta(y_i)\big)\big), \forall i=1,\ldots K.$ Based on the intuition that semantically similar inputs are expected to be close to each other in the embedding space, we, inspired by prior works~\cite{42,18,43}, take $\Omega(\mathbf{z})=\ell_2(\mathbf{z}+\sigma\cdot\boldsymbol{\epsilon})$ where $\ell_2(\cdot)$ is the $\ell_2$ normalization and $\boldsymbol{\epsilon}\sim\mathcal{N}(0,\mathbf{I}_d)$ with $\mathbf{I}_d$ as a $d\times d$ identity matrix.

\textbf{Step 3.} 
The practical form of our proposed debiased OOD scoring function is given as follows:
\begin{equation}
\label{eq17}
\begin{split}
    & \hat{S}_{\text{debiased}}(\mathbf{x};f) = \frac{\frac{1-\tau}{\lambda}\sum_{i=1}^K e^{h(\mathbf{x},y_i)}}{\frac{1-\tau}{\lambda}\sum_{i=1}^K e^{h(\mathbf{x},y_i)}+\frac{1}{L}\sum_{i=1}^{L}e^{h(\mathbf{x},\tilde{y}_{(i)})}-\frac{\tau}{K}\sum_{i=1}^{K}e^{\textsl{g}(\mathbf{x},y_i)}},
\end{split}
\end{equation}
where $\textsl{g}(\mathbf{x},y_j)\triangleq \kappa\cdot f_{\mathcal{X}}(\mathbf{x})^\top \Omega\big(f_{\mathcal{T}}\big(\Delta(y_i)\big)\big)$. Finally, we leverage the grouping strategy to divide the filtered wild labels $\left\{\tilde{y}_{(1)},\ldots,\tilde{y}_{(L)}\right\}$ into $B$ groups (any remaining wild labels are discarded), followed by averaging the debiased OOD score in Eq. (\ref{eq17}) across groups, i.e.,
\begin{equation}
\label{eq18}
\begin{split}
    & \hat{S}_{\text{ours}}(\mathbf{x};f) = \frac{1}{B}\sum_{b=1}^{B}\frac{\frac{1-\tau}{\lambda}\sum_{i=1}^K e^{h(\mathbf{x},y_i)}}{\frac{1-\tau}{\lambda}\sum_{i=1}^K e^{h(\mathbf{x},y_i)}+\frac{1}{\left |\mathcal{G}'_b  \right |}\sum_{\tilde{y}\in\mathcal{G}'_b}e^{h(\mathbf{x},\tilde{y})}-\frac{\tau}{K}\sum_{i=1}^{K}e^{\textsl{g}(\mathbf{x},y_i)}},
\end{split}
\end{equation}
where $\{\tilde{y}_{(i)}\}_{i=1}^L=\bigcup_{b=1}^B\mathcal{G}_b'$ with ${\mathcal{G}}_i'\bigcap{\mathcal{G}}_j'=\emptyset, \forall i \neq j$.

For clarity, we summarize our algorithmic details in Algorithm~\ref{alg1}.
\begin{algorithm}[htb]
   \caption{}
   \label{alg1}
\begin{algorithmic}
\Require Test-time input $\mathbf{x}$, Pre-trained CLIP-based model $f$, ID labels $\mathcal{Y}_{\mathrm{I}}$, unlabeled wild corpus dataset $\mathcal{D}$
\Ensure OOD scoring $\hat{S}_{\text{ours}}(\mathbf{x};f)$
\newline\Comment{{\textit{Step 1}}}
\For{$\tilde{y}_i\in\mathcal{D}$}
\State Computing $Rep(\tilde{y}_i)$ via Eq.(\ref{eq14})
\EndFor
\State Reordering $\mathcal{D}$ according to the decreasing value of $Rep(\tilde{y}_i)$ to have $\hat{\mathcal{D}}=\left\{\tilde{y}_{(1)},\ldots,\tilde{y}_{(T)}\right\}$
\newline\Comment{{\textit{Step 2}}}
\For{$y_i\in\mathcal{Y}_{\mathrm{I}}$}
\State Sampling $\boldsymbol{\epsilon}\sim\mathcal{N}(0,\mathbf{I}_d)$
\State $\Omega\Big(f_{\mathcal{T}}\big(\Delta(y_i)\big)\Big)=\ell_2\Big(f_{\mathcal{T}}\big(\Delta(y_i)\big)+\sigma\cdot\boldsymbol{\epsilon}\Big)$
\EndFor
\newline\Comment{{\textit{Step 3}}}
\State Computing $\hat{S}_{\text{ours}}(\mathbf{x};f)$ via Eq. (\ref{eq18})
\end{algorithmic}
\end{algorithm}

\begin{table*}[tb]
\centering
\caption{OOD detection results on ImageNet-1K as ID, where a VIT B/16 CLIP encoder is adopted. $\uparrow$ indicates larger values are better and vice versa. The best results in the last two columns are shown in bold.} 
\label{imagenet}
\resizebox{1.0 \textwidth}{!}{%
\begin{tabular}{l|cc|cc|cc|cc|cc}
\toprule
Dataset & \multicolumn{2}{c|}{iNaturalist} & \multicolumn{2}{c|}{Sun} & \multicolumn{2}{c|}{Places} & \multicolumn{2}{c|}{Textures} & \multicolumn{2}{c}{Average} \\ 
\midrule
Metric &  AUROC$\uparrow$ &  FPR95$\downarrow$&  AUROC$\uparrow$ &  FPR95$\downarrow$&  AUROC$\uparrow$ &  FPR95$\downarrow$&  AUROC$\uparrow$ &  FPR95$\downarrow$ &  AUROC$\uparrow$ &  FPR95$\downarrow$        \\
 \midrule
\rowcolor{gray!40} \multicolumn{11}{c}{\textbf{Methods requiring training (or fine-tuning)}} \\
MSP &    87.44 & 58.36 & 79.73 & 73.72 & 79.67 & 74.41 & 79.69 & 71.93 & 81.63 & 69.61   \\
ODIN &    94.65 & 30.22 & 87.17 & 54.04 & 85.54 & 55.06 & 87.85 & 51.67 & 88.80 & 47.75   \\
Energy &    95.33 & 26.12 & 92.66 & 35.97 & 91.41 & 39.87 & 86.76 & 57.61 & 91.54 & 39.89 \\
GradNorm &   72.56 & 81.50 & 72.86 & 82.00 & 73.70 & 80.41 & 70.26 & 79.36 & 72.35 & 80.82 \\
ViM  &    93.16 & 32.19 & 87.19 & 54.01 & 83.75 & 60.67 & 87.18 & 53.94 & 87.82 & 50.20 \\
KNN  &    94.52 & 29.17 & 92.67 & 35.62 & 91.02 & 39.61 & 85.67 & 64.35 & 90.97 & 42.19 \\
VOS  &    94.62 & 28.99 & 92.57 & 36.88 & 91.23 & 38.39 & 86.33 & 61.02 & 91.19 & 41.32 \\
NPOS  &    96.19 & 16.58 & 90.44 & 43.77 & 89.44 & 45.27 & 88.80 & 46.12 & 91.22 & 37.93\\
LSN   & 95.83 & 21.56 & 94.35 & 26.32 & 91.25 & 34.48 & 90.42 & 38.54 & 92.96 & 30.22 \\
CLIPN & 95.27 & 23.94 & 93.93 & 26.17 & 92.28 & 33.45 & 90.93 & 40.83 & 93.10 & 31.10 \\
LoCoOp & 96.86 & 16.05 & 95.07 & 23.44 & 91.98 & 32.87 & 90.19 & 42.28 & 93.52 & 28.66 \\
LAPT & 99.63 & 1.16 & 96.01 & 19.12 & 92.01 & 33.01 & 91.06 & 40.32 & 94.68 & 23.40 \\
NegPrompt & 98.73 & 6.32 & 95.55 & 22.89 & 93.34 & 27.60 & 91.60 & 35.21 & 94.81 & 23.01 \\

\midrule
\rowcolor{gray!40}  \multicolumn{11}{c}{\textbf{Zero-Shot Training-free Methods}} \\
Mahalanobis & 55.89 & 99.33 & 59.94 & 99.41 & 65.96 & 98.54 & 64.23 & 98.46 & 61.50 & 98.94 \\
Energy          & 85.09 & 81.08 & 84.24 & 79.02 & 83.38 & 75.08 & 65.56 & 93.65 & 79.57 & 82.21 \\
ZOC   & 86.09 & 87.30 & 81.20 & 81.51 & 83.39 & 73.06 & 76.46 & 98.90 & 81.79 & 85.19 \\
MCM & 94.59 & 32.20 & 92.25 & 38.80 & 90.31 & 46.20 & 86.12 & 58.50 & 90.82 & 43.93 \\
NegLabel &99.49&1.91&95.49& 20.53 & 91.64 & 35.59 & 90.22 & 43.56 & 94.21 & 25.40 \\
\rowcolor{LightCyan} {Ours} & {99.51} & {1.78} & {95.88} & {16.90} & {91.95} & {32.13} & {90.72} & {38.67} & \textbf{94.52} & \textbf{22.37} \\
\bottomrule
\end{tabular}
}
\end{table*}

\section{Experiments}
The goal of our experimental evaluation in this section focuses on answering the following research questions: 
\begin{itemize}[leftmargin=1em]
    \item \textbf{RQ1:} How effective is our method on identifying OOD data?
    \item \textbf{RQ2:} How do various choices of algorithmic designs affect the performance of our method?
    \item \textbf{RQ3:} How sensitive is our method to hyper-parameters?
    \item \textbf{RQ4:} Can our method can be extended to the domain-shifted cases?
\end{itemize}

\textbf{Evaluation Metrics}. The performance of OOD detection is evaluated via two widely used metrics: 1) the false positive rate of OOD data is measured when the true positive rate of ID data reaches 95$\%$ (FPR95); 2) the area under the receiver operating characteristic curve (AUROC) is computed to quantify the probability of the ID case receiving a higher score than the OOD case. The reported results of our method are averaged over 5 independent runs.

\textbf{Baselines}. Our method is compared with advanced methods including MSP~\cite{10}, ODIN~\cite{11}, Energy~\cite{22}, Gradnorm~\cite{13}, Vim~\cite{44}, KNN~\cite{14}, VOS~\cite{45}, NPOS~\cite{46}, ZOC~\cite{26}, CLIPN~\cite{47}, LoCoOp~\cite{48}, LSN~\cite{49}, LAPT~\cite{50}, NegPrompt~\cite{51}, Mahalanobis~\cite{15}, MCM~\cite{27}, NegLabel~\cite{28}. 

\textbf{Datasets}. Following prior work~\cite{27,28}, we mainly evaluate our method on the popular ImageNet-1K benchmark~\cite{52}, where the validation set of ImageNet-1K is designated as the ID dataset while iNaturalist~\cite{53}, SUN~\cite{54}, Places365~\cite{55}, and Textures~\cite{56} are considered as OOD datasets.

\begin{table*}[h]
\begin{center}
\caption{OOD detection results with different CLIP architectures on ImageNet-1k as ID. $\uparrow$ indicates larger values are better and vice versa. The best results in the last two columns are shown in bold. }
\resizebox{\textwidth}{!}{%
\begin{tabular}{ll cc cc cc cc cc}
\toprule
\multicolumn{1}{l}{\bf \small{Backbone}}
&\multicolumn{1}{l}{\bf \small{Method}}
&\multicolumn{2}{c}{\bf \small{iNaturalist}}
&\multicolumn{2}{c}{\bf \small{SUN}} 
&\multicolumn{2}{c}{\bf \small{Places}} 
&\multicolumn{2}{c}{\bf \small{Textures}} 
&\multicolumn{2}{c}{\bf \small{Average}}\\

&
&\small{AUROC$\uparrow$}&\small{FPR95}$\downarrow$ &\small{AUROC}$\uparrow$&\small{FPR95}$\downarrow$ &\small{AUROC}$\uparrow$&\small{FPR95}$\downarrow$ &\small{AUROC}$\uparrow$&\small{FPR95}$\downarrow$ &\small{AUROC}$\uparrow$&\small{FPR95}$\downarrow$ \\
\midrule
\multirow{2}{*}{ViT-B/32}
&MCM &92.68 &40.49 &89.95 &47.83 &88.10 &51.47 &85.98 &60.04 &89.96 &49.96\\
&NegLabel &99.11 &3.73 &95.27 &22.48 &91.72 &34.94 &88.57 &50.51  &93.67 &27.92\\
\rowcolor{LightCyan} &Ours &99.23 &3.24 &95.63 &19.15 &91.84 &32.74 &89.44 &44.56 &\textbf{94.03} &\textbf{24.92} \\
\midrule
\multirow{2}{*}{ViT-L/14}
&MCM &93.58 &36.80 &92.80 &36.77 &90.90 &41.35 &85.05 &61.70  &90.58 &44.16\\
&NegLabel &99.53 &1.77 &95.63 &22.33 &93.01 &32.22 &89.71 &42.92 &94.47 &24.81\\
\rowcolor{LightCyan} &Ours &99.59 &1.64 &95.94 &19.25 &93.03 &31.04 &90.21 &40.67 &\textbf{94.69} &\textbf{23.15}\\
\midrule
\multirow{2}{*}{ResNet50}
&MCM &91.88 &42.97 &89.31 &52.84 &84.12 &65.75 &85.55 &62.15 &87.71 &55.93\\
&NegLabel &99.24 &2.88 &94.54 &26.51 &89.72 &42.60 &88.40 &50.80 &92.97 &30.70\\
\rowcolor{LightCyan} &Ours &99.32 &2.30 &94.75 &23.48 &90.12 &42.55 &89.04 &46.53 &\textbf{93.31} &\textbf{28.72} \\
\bottomrule
\end{tabular}
}
\label{archi}
\end{center}
\end{table*}

\begin{table*}[h]
\begin{center}
\caption{OOD detection results with 336$\times$336 resolution on ImageNet-1k as ID, where a VIT L/14 CLIP encoder is adopted. $\uparrow$ indicates larger values are better and vice versa. The best results in the last two columns are shown in bold. }
\resizebox{\textwidth}{!}{%
\begin{tabular}{l cc cc cc cc cc}
\toprule
\multicolumn{1}{l}{\bf \small{Method}}
&\multicolumn{2}{c}{\bf \small{iNaturalist}}
&\multicolumn{2}{c}{\bf \small{SUN}} 
&\multicolumn{2}{c}{\bf \small{Places}} 
&\multicolumn{2}{c}{\bf \small{Textures}} 
&\multicolumn{2}{c}{\bf \small{Average}}\\
&\small{AUROC$\uparrow$}&\small{FPR95}$\downarrow$ &\small{AUROC}$\uparrow$&\small{FPR95}$\downarrow$ &\small{AUROC}$\uparrow$&\small{FPR95}$\downarrow$ &\small{AUROC}$\uparrow$&\small{FPR95}$\downarrow$ &\small{AUROC}$\uparrow$&\small{FPR95}$\downarrow$ \\
\midrule
NegLabel &99.71 &1.12 &95.68 &21.84 &93.15 &31.79 &90.55 &40.46  &94.77&23.80 \\
\rowcolor{LightCyan} Ours &99.73 &1.05 &96.29 &17.04 &93.16 &30.06 &90.84 &38.21 &\textbf{95.01} &\textbf{21.59} \\
\bottomrule
\end{tabular}
}
\label{input}
\end{center}
\end{table*}

\begin{table*}[h]
\begin{center}
\caption{OOD detection results with different corpus sources on ImageNet-1k as ID, where a VIT B/16 CLIP encoder is adopted. $\uparrow$ indicates larger values are better and vice versa. The best results in the last two columns are shown in bold. }
\resizebox{\textwidth}{!}{%
\begin{tabular}{ll cc cc cc cc cc}
\toprule
\multicolumn{1}{l}{\bf \small{Corpus Sources}}
&\multicolumn{1}{l}{\bf \small{Method}}
&\multicolumn{2}{c}{\bf \small{iNaturalist}}
&\multicolumn{2}{c}{\bf \small{SUN}} 
&\multicolumn{2}{c}{\bf \small{Places}} 
&\multicolumn{2}{c}{\bf \small{Textures}} 
&\multicolumn{2}{c}{\bf \small{Average}}\\

&
&\small{AUROC$\uparrow$}&\small{FPR95}$\downarrow$ &\small{AUROC}$\uparrow$&\small{FPR95}$\downarrow$ &\small{AUROC}$\uparrow$&\small{FPR95}$\downarrow$ &\small{AUROC}$\uparrow$&\small{FPR95}$\downarrow$ &\small{AUROC}$\uparrow$&\small{FPR95}$\downarrow$ \\
\midrule
{Common}
&NegLabel &86.91 &65.43 &95.03 &24.22 &91.52 &34.83 &83.69 &67.75 &89.29 &48.06\\
\rowcolor{LightCyan} &Ours &88.86 &54.94 &95.87 &19.99 &92.51 &30.37 &84.12 &65.83 &\textbf{90.34} &\textbf{42.78} \\
\midrule
{Part-of-Speech}
&NegLabel &99.23 &3.25 &94.20 &25.93 &90.17 &43.09 &87.77 &50.11 &92.84 &30.59\\
\rowcolor{LightCyan} &Ours &99.54 &2.16 &94.94 &22.35 &90.75 &40.38 &90.09 &42.79 &\textbf{93.83} &\textbf{26.92} \\
\bottomrule
\end{tabular}
}
\label{corpus}
\end{center}
\end{table*}

\textbf{Implementation}. Unless otherwise specified, we employ CLIP-B/16 for zero-shot OOD detection. Following prior works~\cite{28,40}, we adopt the text prompt of 'The nice $<$label$>$.' and simulate the empirical wild dataset $\mathcal{D}$ with WordNet~\cite{29}. Notably, we show in Section~\ref{ablation} that our method can generalize well to various CLIP architectures and corpus sources. 
Regarding hyper-parameters, we set $\kappa=0.01$, $B=100$, $\tau=0.5$, $\sigma=0.001$, $L=12000$, and $\alpha=100$. 
\subsection{Main Results (RQ1)}
We compare our method with existing OOD detection methods on the ImageNet-1k benchmark organized in Table~\ref{imagenet}. The methods listed in the upper section of Table~\ref{imagenet}, ranging from MSP~\citep{10} to VOS~\citep{46}, represent traditional visual OOD detection methods. 
Conversely, the methods in the lower section, extending from ZOC~\citep{26} to NegLabel~\citep{28}, employ pre-trained VLMs like CLIP.
Our method  achieves the state-of-the-art on the ImageNet-1k benchmark, which highlights its superior performance in the zero-shot setting. Furthermore, our method can surpass traditional methods with a finetuned CLIP, demonstrating CLIP's strong OOD detection capabilities in zero-shot scenarios. This is because CLIP can parse images in a fine-grained manner, which is achieved through its pre-training on a large-scale image-text dataset.

\begin{figure*}[h]
  \centering
\includegraphics[width=\linewidth]{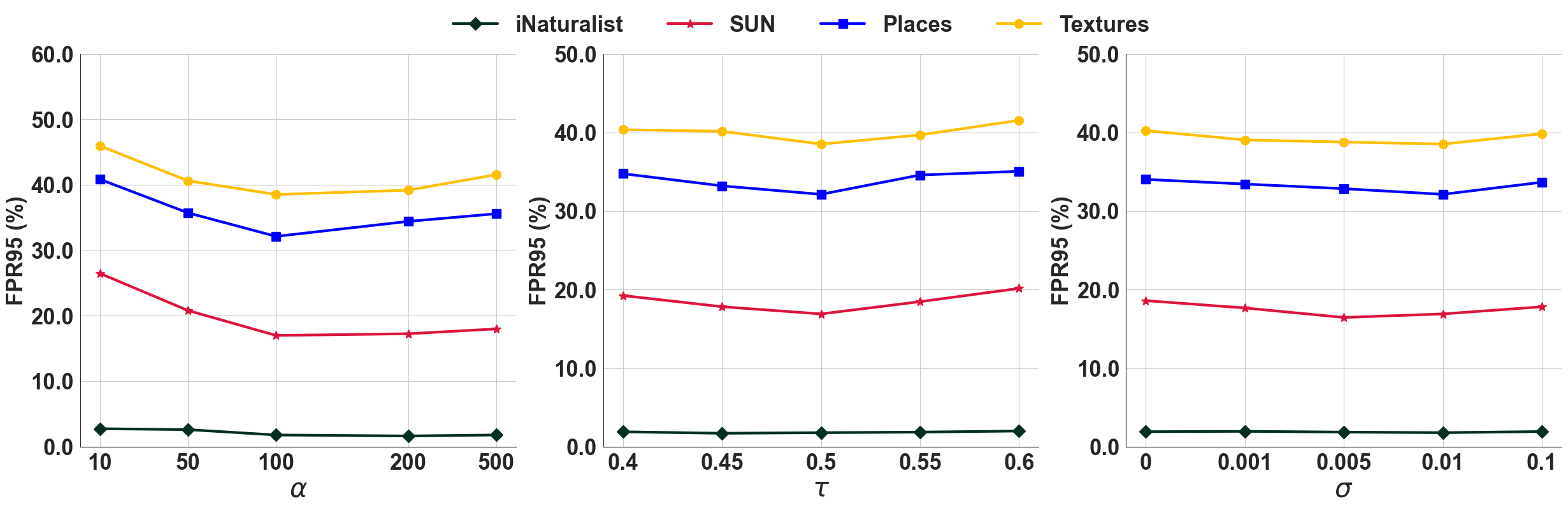}
  \caption{Hyper-parameter analysis on ImageNet-1K w.r.t. $\alpha$ (left), $\tau$ (middle), and $\sigma$ (right).}
\label{hyper}
\end{figure*}
\begin{figure}[h]
  \centering
\includegraphics[width=0.8\linewidth]{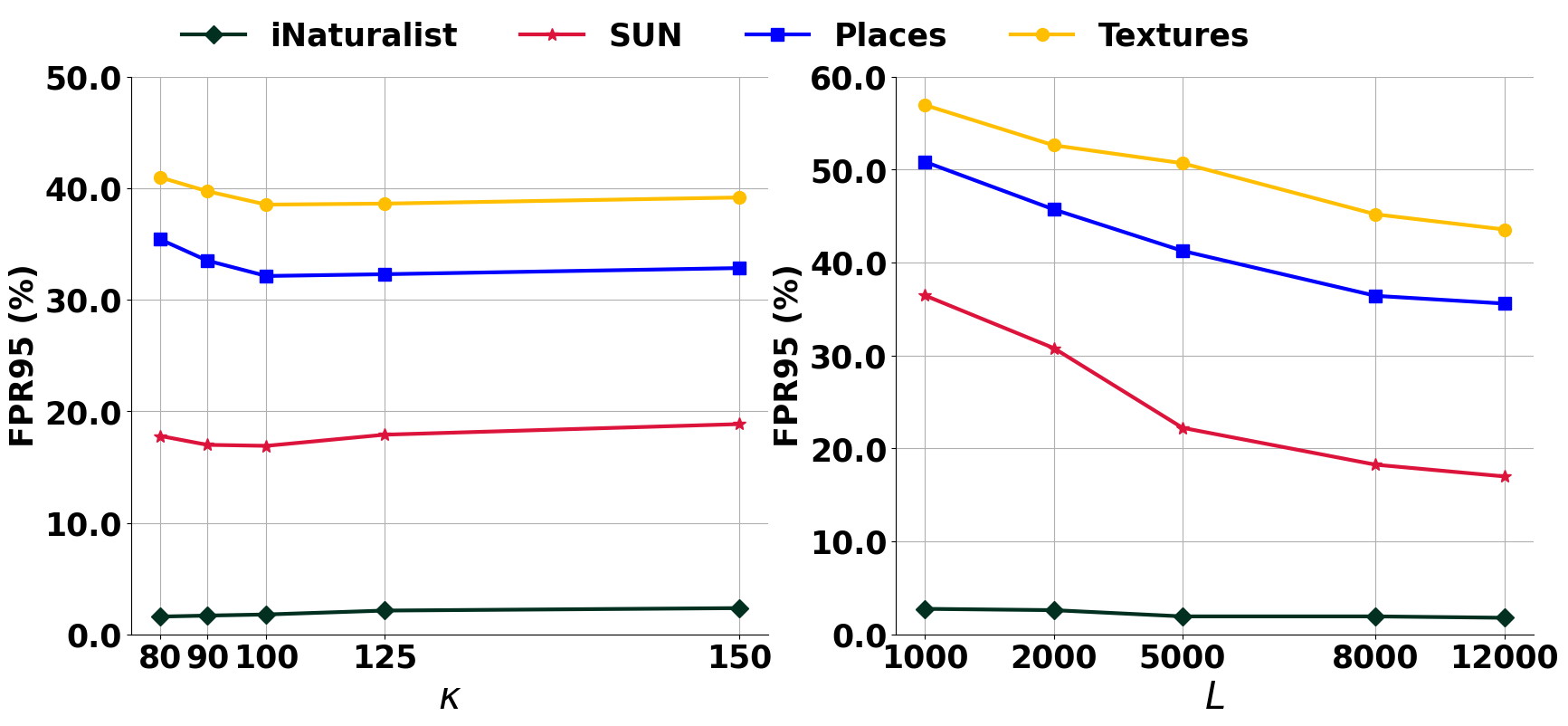}
  \caption{Hyper-parameter analysis results on ImageNet-1K w.r.t. $\kappa$ (left) and $L$ (right).}
\label{hyper2}
\end{figure}

\subsection{Ablation Study (RQ2)}
\label{ablation}
\subsubsection{CLIP Architectures} 
In principle, our method is generic to the choice of visual encoder. We evaluate our method with different visual encoder architectures, including ViT-B/32, ViT-L/14 and ResNet-50, and report the corresponding OOD detection results in Table~\ref{archi}. On the one hand, the performance of OOD detection can be enhanced by more powerful visual encoders. On the other hand, our method consistently outperforms the most recent NegLabel regardless of the backbone architecture used, which implies the better generalization of our method over NegLabel.
\subsubsection{Input Size.}
In principle, our method is generic to the input resolution. We evaluate our method with a larger input size, i.e., 336$\times$336, and report the corresponding OOD detection results in Table~\ref{input}. On the one hand, the performance of OOD detection can be enhanced by a larger input size. On the other hand, our method consistently outperforms the most recent NegLabel regardless of the input resolution, which implies the better generalization of our method over NegLabel.
\subsubsection{Corpus Sources} The role of the corpus is to provide a larger and more comprehensive semantic space. While our method is generic to the input resolution, we also conduct ablative analysis with different corpus sources, including Part-of-Speech Tags\footnote{https://www.kaggle.com/datasets/ruchi798/part-of-speech-tagging} and Common-20K\footnote{https://github.com/first20hours/google-10000-english}. As for Part-of-Speech Tags, we, following NegLabel~\cite{28}, randomly sample $T=70000$ words to simulate the wild dataset $\mathcal{D}$. It can be found that Table~\ref{corpus} that our method consistently outperforms NegLabel on multiple corpora, which implies that the flexibility of our method.
\subsection{Hyper-parameter Analysis (RQ3)} 
We evaluate the hyper-parameters most essential to our design, including $\alpha$ in Eq. (\ref{eq14}), prior probability $\tau$ in Eq. (\ref{eq5}), and the intensity of gaussian noise in positive sampling $\sigma$. The corresponding results are plotted in Figure~\ref{hyper}. It can found that our method is stable across a wide range of the hyper-parameters.
Besides, for completeness, we also evaluate the hyper-parameters orthogonal to to our design, including the temperature $\kappa$ and the length of sampled wild corpus labels. The corresponding results are plotted in Figure~\ref{hyper2}. Firstly, having a large or small value of $\kappa$ does not necessarily improve the performance, which is consistent with the observations in prior works~\cite{28} . Besides, OOD detection performance improves as $L$ increases, which empirically epochs our analysis in Theorem~\ref{thm2}.

\begin{table*}[h]
\begin{center}
\caption{OOD detection performance robustness to domain shift. All methods are based on CLIP-B/16. $\uparrow$ indicates larger values are better and vice versa. The best results in the last two columns are shown in bold.}
\resizebox{\textwidth}{!}{%
\begin{tabular}{ll cc cc cc cc cc}
\toprule
\multicolumn{1}{l}{\bf ID Dataset}
&\multicolumn{1}{l}{\bf \small{Method}}
&\multicolumn{2}{c}{\bf \small{iNaturalist}}
&\multicolumn{2}{c}{\bf \small{SUN}} 
&\multicolumn{2}{c}{\bf \small{Places}} 
&\multicolumn{2}{c}{\bf \small{Textures}} 
&\multicolumn{2}{c}{\bf \small{Average}}\\

&
&\small{AUROC$\uparrow$}&\small{FPR95}$\downarrow$ &\small{AUROC}$\uparrow$&\small{FPR95}$\downarrow$ &\small{AUROC}$\uparrow$&\small{FPR95}$\downarrow$ &\small{AUROC}$\uparrow$&\small{FPR95}$\downarrow$ &\small{AUROC}$\uparrow$&\small{FPR95}$\downarrow$ \\
\midrule
\multirow{2}{*}{ImageNet-S}
&MCM &87.74 &63.06 &85.35 &67.24 &81.19 &70.64 &74.77 &79.59 &82.26 &70.13\\
&NegLabel &99.34 &2.24 &94.93 &22.73 &90.78 &38.62 &89.29 &46.10  &93.59 &27.42\\
\rowcolor{LightCyan} &Ours &99.41 &1.63 &95.47 &19.82 &90.84 &68.46 &89.87 &44.43 &\textbf{93.90} &\textbf{26.09} \\
\midrule
\multirow{3}{*}{ImageNet-A}
&MCM &79.50 &76.85 &76.19 &79.78 &70.95 &80.51 &61.98 &86.37 &72.16 &80.88 \\
&NegLabel &98.80 &4.09 &89.83 &44.38 &82.88 &60.10 &80.25 &64.34 &87.94 &43.23\\
\rowcolor{LightCyan} &Ours &98.88 &3.39 &91.21 &36.32 &84.28 &55.12 &82.00 &60.14 &\textbf{89.09} &\textbf{38.74} \\
\midrule
\multirow{3}{*}{ImageNet-R}
&MCM &83.22 &71.51 &80.31 &74.98 &75.53 &76.67 &67.66 &83.72 &76.68 &76.72\\
&NegLabel &99.58 &1.60 &96.03 &15.77 &91.97 &29.48 &90.60 &35.67 &94.54 &20.63\\
\rowcolor{LightCyan} &Ours &99.57 &1.57 &96.43 &13.13 &92.27 &26.53 &91.07 &32.54 &\textbf{94.83} &\textbf{18.44} \\
\midrule
\multirow{3}{*}{ImageNetV2}
&MCM &91.79 &45.90 &89.88 &50.73 &86.52 &56.25 &81.51 &69.57 &87.43 &55.61\\
&NegLabel &99.40 &2.47 &94.46 &25.69 &90.00 &42.03 &88.46 &48.90 &93.08 &29.77\\
\rowcolor{LightCyan} &Ours &99.42 &2.08 &95.12 &21.14 &90.45 &38.47 &89.05 &45.00 &\textbf{93.51} &\textbf{26.67} \\
\bottomrule
\end{tabular}
}
\label{DG}
\end{center}
\end{table*}
\subsection{Domain-generalizable OOD Detection (RQ4)}
We investigate domain generalizable OOD detection scenarios, where there exist domain shifts in ID data. With ImageNet-1K as a case study, we, following~\citet{28}, consider ImageNet-A~\cite{69}, ImageNet-R~\cite{71} and ImageNetV2~\cite{70} as ID data receptively. The experiment results on four OOD datasets are shown in Table~\ref{DG}. It is apparent that the zero-shot OOD detection performance of MCM significantly deteriorates across diverse domain shifts, indicating the difficulty of zero-shot OOD detection under such conditions. NegLabel achieves remarkably better performances than MCM, thus demonstrating the significance of introducing negative labels for OOD detection. Even though, our method consistently outperform NegLabel cross diverse ID datasets, which implies stronger robustness of our method against domain shifts.

\section{Related work}
\subsection{Traditional Out-of-distribution Detection} 
The popularity of OOD detection is motivated by the empirical observation~\cite{12} that neural networks are over-confident in OOD data.
One line of work performs OOD detection by devising post-hoc scoring functions, including confidence-based methods~\cite{23,27,24}, energy-based methods~\cite{22,21}, distance-based approaches~\cite{18,15,14,19,25,16,20}, gradient-based approaches~\cite{13}, and Bayesian approaches~\cite{57,58}.
Another line of work addresses OOD detection by fine-tuning a pre-trained discrimination model with training-time regularizations that help the model learn ID/OOD discrepancy following the guideline of outlier exposure~\cite{59}.
For instance, the discriminative model is regularized to produce lower confidence~\cite{60,62}, smaller feature magnitudes~\cite{22} or higher energy~\cite{61} for outlier points. More recently, some works have considered a practical scenario where the auxiliary outliers can be arbitrarily different from the real OOD data, therefore distributionally augmenting the observed OOD data. Besides, the given OOD samples tend to include unlabeled ID counterparts~\cite{63}. Because of this, WOOD~\cite{63} formulates learning with noisy OOD samples as a constrained optimization problem while SAL~\cite{64} separates candidate outliers from unlabeled wild data and then trains a binary classifier using the candidate outliers and the labeled ID data.


\subsection{CLIP-based Out-of-distribution Detection}

The core of CLIP-based OOD detection lies in how to leverage texture supervision with pre-trained VLMs to assist OOD detection on the visual domain.
On the one hand, the pioneering work, MCM~\cite{27}, defines textual features as concept proto-types for each ID class and uses the scaled distance between visual features and the closest ID prototype to measure OOD uncertainty. Intead of relying  on textual information from only ID label space, ZOC \cite{26} applies VLMs to discern OOD instances by training a captioner that generates potential OOD labels. Nevertheless, this captioner often fails to produce effective OOD labels, particularly for ID datasets containing many classes. Differently, NegLabel~\cite{28} incorporates additional negative class names mined from available data sources as negative proxies. Considering the nonalignment between target visual OOD distribution and the generated negative textual OOD distribution, AdaNeg~\cite{40} leverages the benefits of test-time adaptation to generate adaptive proxies by exploring potential OOD images during testing. 
On the other hand, CLIP-based OOD detection can also be improved by prompt representation learning. In particular, LoCoOp~\cite{48} learns ID text prompts by pushing them away from the portions of CLIP local features that have ID-irrelevant nuisances (e.g., backgrounds). CLIPN~\cite{47} and LSN~\cite{49} design a learnable “no” prompt and a “no” text encoder to capture negation semantics within images. Differently, LAPT~\cite{50} initializes prompts with negative labels~\cite{28}, followed by tuning prompts with cross-modal and cross-distribution mixing. More recently, ~\citet{r} understand CLIP-based post-hoc OOD detection from an information-theoretical perspective.

\section{Conclusion}
The introduction of negative labels has found effective in enhancing OOD detection. However, existing methods often rely on heuristic rules to mine negative labels from unlabeled wild corpus data, which fail to correctly capture semantic relationships, and face difficulties in handling challenging false negatives. In this work, we propose to correct for the bias introduced by the common practice of sampling negative examples from the unreliable negative distribution.
Experimental results show that the proposed method achieves state-of-the-art performance on multiple OOD detection benchmarks and generalizes well across various VLM architectures and setups.
In the future, a interesting direction is to delve into the incorporation of pre-trained Multi-modal Large Language Models. 

\medskip

{
\small
\bibliographystyle{name}
\bibliography{ref}
}

\appendix

\section{Derivation of Theorem~\ref{thm1}}
\label{appendix b}
We plug in the decomposition $\mathbb{P}_Y^{-}=(\mathbb{Q}_Y-\tau\cdot\mathbb{P}_Y^{+})/(1-\tau)$ as follows:
\begin{align*}
    S_{\text{unbiased}}(\mathbf{x};f)
    &= \int\prod_{i=1}^m p_Y^-(\hat{y}_i) \Phi\left (\mathbf{x}, \left \{\hat{y}_i\right \}_{i=1}^m \right ) \prod_{i=1}^m \text{d}\hat{y}_i \\
    &= \int \prod_{i=1}^m \frac{q_Y(\hat{y}_i) - \tau\cdot p_Y^+(\hat{y}_i)}{1-\tau} \Phi\left (\mathbf{x}, \left \{\hat{y}_i\right \}_{i=1}^m \right )  \prod_{i=1}^m \text{d}\hat{y}_i \\
    &= \frac{1}{(1-\tau)^m}\int\prod_{i=1}^m \left(q_Y(\hat{y}_i) - \tau\cdot p_Y^+(\hat{y}_i) \right) \Phi\left (\mathbf{x}, \left \{\hat{y}_i\right \}_{i=1}^m \right ) \prod_{i=1}^m \text{d}\hat{y}_i.
\end{align*}
By the Binomial Theorem, the product inside the integral can be separated into $m+1$ groups corresponding to how many $\hat{y}_i$ are sampled from $\mathbb{Q}_Y$, i.e., 
\begin{align*}
    &\text{(1)}\quad\quad \prod_{i=1}^m q_Y(\hat{y}_i) \\
    &\text{(2)}\quad\quad \binom{m}{1} ( -\tau) p_Y^+(\hat{y}_1) \prod_{i=2}^m q_Y(\hat{y}_i) \\
    &\cdots \nonumber\\
    &\text{($k+1$)}\quad\quad \binom{m}{k} \prod_{j=1}^k (-\tau) p_Y^+(\hat{y}_i) \prod_{i=k+1}^m q_Y(\hat{y}_i) \\
    &\cdots \nonumber\\
    &\text{($m+1$)}\quad\quad \prod_{i=1}^m (-\tau) p_Y^+(\hat{y}_i)
\end{align*}
Finally, the objective $S_{\text{unbiased}}(\mathbf{x};f)$ becomes
\begin{equation*}
S_{\text{unbiased}}(\mathbf{x};f) = \sum_{k=0}^m
\expectunder[\substack{
\{\hat{y}_i\}_{i=1}^k\,\iidsim\,\mathbb{P}_Y^{+}  \\ 
\{\hat{y}_i\}_{i=k+1}^{m}\,\iidsim\,\mathbb{Q}_Y
    }]
{\binom{m}{k}
\frac{(-\tau)^k}{(1-\tau)^m}\cdot\Phi\left (\mathbf{x}, \left \{\hat{y}_i\right \}_{i=1}^m\right )}.
\end{equation*}
Please refer to the Inclusion–exclusion principle for more details. 



\section{Proof of theorem~\ref{thm2}}
\label{appendix c}
As a reminder, let
\begin{equation}
\notag
\delta(\mathbf{x};f)=\log\hat{S}_{\text{unbiased}}(\mathbf{x};f)-\log\hat{S}_{\text{debiased}}(\mathbf{x};f),
\end{equation}
where
\begin{equation}
\notag
\begin{split}
    & \hat{S}_{\text{debiased}}(\mathbf{x};f) \trieq \frac{\frac{1-\tau}{\lambda}\sum_{i=1}^K e^{h(\mathbf{x},y_i)}}{\frac{1-\tau}{\lambda}\sum_{i=1}^K e^{h(\mathbf{x},y_i)}+\frac{1}{m}\sum_{i=1}^{m}e^{h(\mathbf{x},u_i)}-\frac{\tau}{n}\sum_{i=1}^{n}e^{h(\mathbf{x},v_i)}}
\end{split}
\end{equation}
and
\begin{equation}
\notag
\begin{split}
\hat{S}_{\text{unbiased}}(\mathbf{x};f)
\triangleq\frac{\sum_{i=1}^K e^{h(\mathbf{x},y_i)}}{\sum_{i=1}^K e^{h(\mathbf{x},y_i)}+\lambda\mathbb{E}_{\hat{y}\sim\mathbb{P}_Y^{-}}[e^{h(\mathbf{x},{\hat{y}})}]}.
\end{split}
\end{equation}
For simplicity, we define 
$$
\Lambda(\mathbf{x},\{u_i\}_{i=1}^m,\{v_i\}_{i=1}^n)=\frac{1}{m(1-\tau)}\sum_{i=1}^{m}e^{h(\mathbf{x},u_i)}-\frac{\tau}{n(1-\tau)}\sum_{i=1}^{n}e^{h(\mathbf{x},v_i)}
$$ 

We start with decomposing the probability as
\begin{align*}
    &\mathbb{P}\bigg (\bigg | \delta(\mathbf{x};f) \bigg | \geq \zeta\bigg )  
    \\
   =&\mathbb{P} \bigg ( \bigg | \log\frac{\sum_{i=1}^K e^{h(\mathbf{x},y_i)}+\lambda\Lambda(\mathbf{x},\{u_i\}_{i=1}^m,\{v_i\}_{i=1}^n)}{\sum_{i=1}^K e^{h(\mathbf{x},y_i)}+\lambda\mathbb{E}_{\hat{y}\sim\mathbb{P}_Y^{-}}[e^{h(\mathbf{x},{\hat{y}})}]}  \bigg | \geq \zeta \bigg ) 
    \\
   =&\mathbb{P} \bigg ( \log\frac{\sum_{i=1}^K e^{h(\mathbf{x},y_i)}+\lambda\Lambda(\mathbf{x},\{u_i\}_{i=1}^m,\{v_i\}_{i=1}^n)}{\sum_{i=1}^K e^{h(\mathbf{x},y_i)}+\lambda\mathbb{E}_{\hat{y}\sim\mathbb{P}_Y^{-}}[e^{h(\mathbf{x},{\hat{y}})}]}  \geq \zeta \bigg )
    \\
    \quad&+ \mathbb{P} \bigg ( -\log\frac{\sum_{i=1}^K e^{h(\mathbf{x},y_i)}+\lambda\Lambda(\mathbf{x},\{u_i\}_{i=1}^m,\{v_i\}_{i=1}^n)}{\sum_{i=1}^K e^{h(\mathbf{x},y_i)}+\lambda\mathbb{E}_{\hat{y}\sim\mathbb{P}_Y^{-}}[e^{h(\mathbf{x},{\hat{y}})}]} \geq \zeta \bigg ),
\end{align*}
where the final equality holds simply because $|X| \geq \varepsilon$ if and only if $X \geq \varepsilon$ or $-X \geq \varepsilon$. 

The first term can be bounded as
\begin{align}
&\mathbb{P} \bigg ( \log\frac{\sum_{i=1}^K e^{h(\mathbf{x},y_i)}+\lambda\Lambda(\mathbf{x},\{u_i\}_{i=1}^m,\{v_i\}_{i=1}^n)}{\sum_{i=1}^K e^{h(\mathbf{x},y_i)}+\lambda\mathbb{E}_{\hat{y}\sim\mathbb{P}_Y^{-}}[e^{h(\mathbf{x},{\hat{y}})}]}  \geq \zeta \bigg ) \nonumber \\
\leq& \mathbb{P} \bigg (\frac{\lambda\Lambda(\mathbf{x},\{u_i\}_{i=1}^m,\{v_i\}_{i=1}^n)-\lambda\mathbb{E}_{\hat{y}\sim\mathbb{P}_Y^{-}}[e^{h(\mathbf{x},{\hat{y}})}]}{\sum_{i=1}^K e^{h(\mathbf{x},y_i)}+\lambda\mathbb{E}_{\hat{y}\sim\mathbb{P}_Y^{-}}[e^{h(\mathbf{x},{\hat{y}})}]} \geq \zeta \bigg ) \nonumber \\
=& \mathbb{P} \bigg (\frac{\Lambda(\mathbf{x},\{u_i\}_{i=1}^m,\{v_i\}_{i=1}^n)-\mathbb{E}_{\hat{y}\sim\mathbb{P}_Y^{-}}[e^{h(\mathbf{x},{\hat{y}})}]}{\frac{1}{\lambda}\sum_{i=1}^K e^{h(\mathbf{x},y_i)}+\mathbb{E}_{\hat{y}\sim\mathbb{P}_Y^{-}}[e^{h(\mathbf{x},{\hat{y}})}]} \geq \zeta \bigg ) \nonumber \\
\leq &\mathbb{P} \bigg (\Lambda(\mathbf{x},\{u_i\}_{i=1}^m,\{v_i\}_{i=1}^n)-\mathbb{E}_{\hat{y}\sim\mathbb{P}_Y^{-}}[e^{h(\mathbf{x},{\hat{y}})}]\geq \zeta e^{-\kappa}\bigg ).\label{a_eq_1}
\end{align}
The first inequality follows by applying the fact that  $\log x \leq x -1 $ for $x > 0$. The second inequality holds since 
$$ \frac{1}{\lambda}\sum_{i=1}^K e^{h(\mathbf{x},y_i)}+\mathbb{E}_{\hat{y}\sim\mathbb{P}_Y^{-}}[e^{h(\mathbf{x},{\hat{y}})}]\geq \mathbb{E}_{\hat{y}\sim\mathbb{P}_Y^{-}}[e^{h(\mathbf{x},{\hat{y}})}]\geq e^{-\kappa}.$$

Next, we move on to bounding the second term, which proceeds similarly, using the same two bounds.
\begin{align}
&\mathbb{P} \bigg ( -\log\frac{\sum_{i=1}^K e^{h(\mathbf{x},y_i)}+\lambda\Lambda(\mathbf{x},\{u_i\}_{i=1}^m,\{v_i\}_{i=1}^n)}{\sum_{i=1}^K e^{h(\mathbf{x},y_i)}+\lambda\mathbb{E}_{\hat{y}\sim\mathbb{P}_Y^{-}}[e^{h(\mathbf{x},{\hat{y}})}]} \geq \zeta \bigg ) \nonumber\\
=& \mathbb{P} \bigg ( \log\frac{\sum_{i=1}^K e^{h(\mathbf{x},y_i)}+\lambda\mathbb{E}_{\hat{y}\sim\mathbb{P}_Y^{-}}[e^{h(\mathbf{x},{\hat{y}})}]}{\sum_{i=1}^K e^{h(\mathbf{x},y_i)}+\lambda\Lambda(\mathbf{x},\{u_i\}_{i=1}^m,\{v_i\}_{i=1}^n)} \geq \zeta \bigg ) \nonumber \\
\leq& \mathbb{P} \bigg (\frac{ \lambda\mathbb{E}_{\hat{y}\sim\mathbb{P}_Y^{-}}[e^{h(\mathbf{x},{\hat{y}})}] - \lambda\Lambda(\mathbf{x},\{u_i\}_{i=1}^m,\{v_i\}_{i=1}^n) }{\sum_{i=1}^K e^{h(\mathbf{x},y_i)}+\lambda\Lambda(\mathbf{x},\{u_i\}_{i=1}^m,\{v_i\}_{i=1}^n)} \geq \zeta \bigg ) \nonumber \\
=& \mathbb{P} \bigg (\frac{\mathbb{E}_{\hat{y}\sim\mathbb{P}_Y^{-}}[e^{h(\mathbf{x},{\hat{y}})}] - \Lambda(\mathbf{x},\{u_i\}_{i=1}^m,\{v_i\}_{i=1}^n)}{\frac{1}{\lambda}\sum_{i=1}^K e^{h(\mathbf{x},y_i)}+\Lambda(\mathbf{x},\{u_i\}_{i=1}^m,\{v_i\}_{i=1}^n)} \geq \zeta \bigg ) \nonumber \\
\leq& \mathbb{P} \bigg ( \mathbb{E}_{\hat{y}\sim\mathbb{P}_Y^{-}}[e^{h(\mathbf{x},{\hat{y}})}] - \Lambda(\mathbf{x},\{u_i\}_{i=1}^m,\{v_i\}_{i=1}^n)  \geq \zeta e^{-\kappa} \bigg ) . \label{a_eq_2}
\end{align} 
Combining equation \eqref{a_eq_1} and equation \eqref{a_eq_2}, we have
\begin{align*}
\mathbb{P}(\delta_3 \geq \zeta) \leq \mathbb{P} \bigg ( \Big |\Lambda(\mathbf{x},\{u_i\}_{i=1}^m,\{v_i\}_{i=1}^n)-\mathbb{E}_{\hat{y}\sim\mathbb{P}_Y^{-}}[e^{h(\mathbf{x},{\hat{y}})}] \Big |  \geq \zeta e^{-\kappa} \bigg ).
\end{align*}
We then proceed to bound the right hand tail probability. Based on the triangle inequality, i.e., 
\begin{equation*}
\small
\begin{split}
&\Big |\Lambda(\mathbf{x},\{u_i\}_{i=1}^m,\{v_i\}_{i=1}^n)-\mathbb{E}_{\hat{y}\sim\mathbb{P}_Y^{-}}[e^{h(\mathbf{x},{\hat{y}})}] \Big | \\
=& \Big |\frac{1}{m(1-\tau)}\sum_{i=1}^{m}e^{h(\mathbf{x},u_i)}-\frac{\tau}{n(1-\tau)}\sum_{i=1}^{n}e^{h(\mathbf{x},v_i)}-\frac{1-\tau}{1-\tau}\mathbb{E}_{\hat{y}\sim\mathbb{P}_Y^{-}}[e^{h(\mathbf{x},{\hat{y}})}]\Big |\\
\leq& \frac{1}{1-\tau} \bigg |\frac{1}{m}\sum_{i=1}^{m}e^{h(\mathbf{x},u_i)}  - \mathbb{E}_{\hat{y}\sim\mathbb{Q}_Y}[e^{h(\mathbf{x},{\hat{y}})}] \bigg | \\
&\qquad\qquad\qquad\qquad\qquad +\frac{\tau}{1-\tau} \bigg  |\frac{1}{n}\sum_{i=1}^{n}e^{h(\mathbf{x},v_i)}-\mathbb{E}_{\hat{y}\sim\mathbb{P}_Y^{+}}[e^{h(\mathbf{x},{\hat{y}})}] \bigg |,
\end{split}
\end{equation*}
we find that 
\begin{align*}
\mathbb{P} \bigg ( \Big |\Lambda(\mathbf{x},\{u_i\}_{i=1}^m,\{v_i\}_{i=1}^n)-\mathbb{E}_{\hat{y}\sim\mathbb{P}_Y^{-}}[e^{h(\mathbf{x},{\hat{y}})}] \Big |  \geq \zeta e^{-\kappa} \bigg )  \leq \text{I} (\zeta) + \text{II} (\zeta),
\end{align*}
where 
\begin{align*}
\text{I} (\zeta)=  \mathbb{P} \left (  \frac{1}{1-\tau} \bigg |\frac{1}{m}\sum_{i=1}^{m}e^{h(\mathbf{x},u_i)}  - \mathbb{E}_{\hat{y}\sim\mathbb{Q}_Y}[e^{h(\mathbf{x},{\hat{y}})}] \bigg |   \geq \frac{\zeta e^{-\kappa}}{2} \right )
\end{align*}
and
\begin{align*}
\text{II}(\varepsilon)=  \mathbb{P} \left ( \frac{\tau}{1-\tau} \bigg  |\frac{1}{n}\sum_{i=1}^{n}e^{h(\mathbf{x},v_i)}-\mathbb{E}_{\hat{y}\sim\mathbb{P}_Y^{+}}[e^{h(\mathbf{x},{\hat{y}})}] \bigg | \geq \frac{\zeta e^{-\kappa}}{2} \right ).
\end{align*}
since $e^{h(\mathbf{x},y)}\in[e^{-\kappa},e^\kappa]$ is bounded $\forall y\in\mathcal{Y}$, the Hoeffding's inequality implies the following bound on the tails of both terms:
\begin{align*}
\text{I} (\varepsilon)   \leq  2 \exp \left ( - \frac{m \zeta^2 (1-\tau)^2}{2e^{3\kappa}} \right ) \quad \text{and} \quad \text{II}(\varepsilon)  \leq 2 \exp \left ( - \frac{n \zeta^2 (1-\tau)^2}{2\tau^2e^{3\kappa}} \right ).
\end{align*}

Note that implicitly, $\delta(\mathbf{x};f)$ depends on the collections $\{u_i\}_{i=1}^m$ and $ \{v_i\}_{i=1}^n$, we can rewrite the expectation of $\delta(\mathbf{x};f)$ with regard to the collections $\{u_i\}_{i=1}^m$ and $ \{v_i\}_{i=1}^n$ as the integral of its tail probability, i.e.,
\begin{align*}
\mathbb{E}_{\{u_i\}_{i=1}^m,\{v_i\}_{i=1}^n}\left [\delta(\mathbf{x};f) \right ]
&=\int _0 ^\infty \mathbb{P}(\delta(\mathbf{x};f) \geq \zeta) \text{d} \zeta\\
&\leq \int _0 ^\infty   2 \exp \left ( - \frac{N \zeta^2 (1-\tau)^2}{2e^{3\kappa}} \right )  \text{d} \zeta  + \int _0 ^\infty    2 \exp \left ( - \frac{M \zeta^2 (1-\tau)^2}{2\tau^2e^{3\kappa}} \right )  \text{d} \zeta\\
&= \frac{1}{1-\tau}\sqrt{\frac{\pi e^{3\kappa}}{2m}} + \frac{\tau}{1-\tau}\sqrt{\frac{\pi e^{3\kappa}}{2n}},
\end{align*}
where the last step is derived based on the classical identity
$$
\int_0^\infty e^{-c z^2} \text{d}z = \frac{1}{2}\sqrt{ \frac{\pi}{c}}.
$$

\end{document}